\documentclass[]{article}
\usepackage{fullpage,amssymb,amsmath,color,graphicx}
\usepackage{amsmath,amsfonts,amsthm,amssymb,xspace,bm}
\usepackage{verbatim,dsfont,mathtools,algorithm,algorithmic,url}
\usepackage{booktabs}

\usepackage{multirow}
\usepackage{xfrac}
\usepackage[numbers]{natbib}
\usepackage{wrapfig}
\usepackage{varwidth}
\usepackage{colortbl}
\usepackage{epstopdf}
\usepackage{appendix}
\usepackage{enumerate}
\usepackage{pifont}
\usepackage[table]{xcolor}
\usepackage{color}

\theoremstyle{plain}
\newtheorem{theorem}{Theorem}[section]
\newtheorem{lemma}[theorem]{Lemma}
\newtheorem*{lemma*}{Lemma}

\newtheorem{proposition}[theorem]{Proposition}
\newtheorem{corollary}[theorem]{Corollary}

\newtheorem{definition}[theorem]{Definition}

\theoremstyle{definition}

\newtheorem{remark}{Remark}

\newcommand{\R}{\mathbb{R}}

\newcommand{\ip}[2]{\left\langle #1, #2 \right\rangle}

\newcommand{\norm}[1]{\left \Vert #1\right \Vert}
\newcommand{\dist}{{\rm{\textsc{Dist}}}}

\newcommand{\normf}[1]{\left\|{#1}\right\|_F}

\def\Xs{{X^\star}}

\def\Us{{U^\star}}
\def\Vs{{V^\star}}

\def\Ws{{W^\star}}
\def\Zt{\tilde{Z}}

\def\X{X}
\def\R{\mathbb{R}}
\def\U{U}
\def\V{V}
\def\Uo{U^\star}
\def\Vo{V^\star}

\newcommand{\Wt}{\tilde{W}}

\begin{document}

\title{Non-square matrix sensing without spurious local minima \\ via the Burer-Monteiro approach} 
\author{Dohyung Park, Anastasios Kyrillidis, Constantine Caramanis, and Sujay Sanghavi \\
 \vspace{0.4cm}
 The University of Texas at Austin \\
 \vspace{0.1cm}
 \{dhpark, anastasios, constantine\}@utexas.edu,
 sanghavi@mail.utexas.edu} 
\maketitle

\begin{abstract}
We consider the \emph{non-square} matrix sensing problem, under restricted isometry property (RIP) assumptions. 
We focus on the non-convex formulation, where any rank-$r$ matrix $\X \in \R^{m \times n}$ is represented as $UV^\top$, where $U \in \R^{m \times r}$ and $V \in \R^{n \times r}$.
In this paper, we complement recent findings on the non-convex geometry of the analogous PSD setting \cite{bhojanapalli2016global}, and show that matrix factorization does not introduce any spurious local minima, under RIP.
\end{abstract}

\section{Introduction and Problem Formulation}
Consider the following matrix sensing problem:
\begin{equation}\label{eq:intro_00}
\begin{aligned}
& \underset{X \in \R^{m \times n}}{\text{minimize}}
& & f(X) := \|\mathcal{A}(X) - b\|_2^2 
& \text{subject to}
& & \text{rank}(X) \leq r.
\end{aligned}
\end{equation} 
Here, $b \in \R^p$ denotes the set of observations and $\mathcal{A}:\R^{m \times n} \rightarrow \R^p$ is the sensing linear map. 
The motivation behind this task comes from several applications, where we are interested in inferring an unknown matrix $X^\star \in \R^{m \times n}$ from $b$. 
Common assumptions are $(i)$ $p \ll m \cdot n$, $(ii)$ $b = \mathcal{A}(X^\star) + w$, \emph{i.e.}, we have a linear measurement system, and $(iii)$ $X^\star$ is rank-$r$, $r \ll \min\{m, n\}$.
Such problems appear in a variety of research fields and include image processing \cite{candes2011robust, waters2011sparcs}, data analytics \cite{chandrasekaran2009sparse, candes2011robust}, quantum computing \cite{aaronson2007learnability, flammia2012quantum, kalev2015quantum}, systems \cite{liu2009interior}, and sensor localization \cite{javanmard2013localization} problems. 

There are numerous approaches that solve \eqref{eq:intro_00}, both in its original non-convex form or through its convex relaxation; see \cite{kyrillidis2014matrix, davenport2016overview} and references therein. 
However, satisfying the rank constraint (or any nuclear norm constraints in the convex relaxation) per iteration requires SVD computations, which could be prohibitive in practice for large-scale settings. 
To overcome this obstacle, recent approaches reside on non-convex parametrization of the variable space and  encode the low-rankness directly into the objective \cite{jain2015computing, anandkumar2016efficient, tu2015low, zheng2015convergent, chen2015fast, bhojanapalli2015dropping, zhao2015nonconvex, sun2015guaranteed, zheng2016convergent, jin2016provable, park2016provable, yi2016rpca, park2016finding, park2016findingb}.
In particular, we know that a rank-$r$ matrix $\X \in \R^{m \times n}$ can be written as a product $\U \V^\top$, where $\U \in \R^{m \times r}$ and $\V \in \R^{n \times r}$. 
Such a re-parametrization technique has a long history \cite{wold1969nonlinear, christoffersson1970one, ruhe1974numerical}, and was popularized by Burer and Monteiro \citep{burer2003nonlinear, burer2005local} for solving semi-definite programs (SDPs).
Using this observation in \eqref{eq:intro_00}, we obtain the following \emph{non-convex, bilinear} problem:
\begin{equation}
\begin{aligned} \label{eqn:formulation}
& \underset{U \in \R^{m \times r}, V \in \R^{n \times r}}{\text{minimize}}
& & f(UV^\top) := \|\mathcal{A}(UV^\top) - b\|_2^2.
\end{aligned}
\end{equation} 
Now, \eqref{eqn:formulation} has a different form of non-convexity due to the bilinearity of the variable space, which raises the question whether we introduce spurious local minima by doing this transformation.

\textit{Contributions:} The goal of this paper is to answer negatively to this question: \emph{We show that, under standard regulatory assumptions on $\mathcal{A}$, $UV^\top$ parametrization does not introduce any spurious local minima.}
To do so, we non-trivially generalize recent developments for the square, PSD case \cite{bhojanapalli2016global} to the non-square case for $\X^\star$. 
Our result requires a different (but equivalent) problem re-formulation and analysis, with the introduction of an appropriate regularizer in the objective. 

\paragraph{Related work.}
There are several papers that consider similar questions, but for other objectives. 
\cite{sun2015complete} characterizes the non-convex geometry of the \emph{complete} dictionary recovery problem, and proves that all local minima are global; 
\cite{boumal2016nonconvex} considers the problem of non-convex phase synchronization where the task is modeled as a non-convex least-squares optimization problem, and can be globally solved via a modified version of power method;
\cite{sun2016geometric} show that a nonconvex fourth-order polynomial objective for phase retrieval has no local minimizers and all global minimizers are equivalent;
\cite{bandeira2016low, boumal2016non} show that the Burer-Monteiro approach works on smooth semidefinite programs, with applications in synchronization and community detection;
\cite{de2014global} consider the PCA problem under streaming settings and use martingale arguments to prove that stochastic gradient descent on the factors reaches to the global solution with non-negligible probability;
\cite{ge2015escaping} introduces the notion of \emph{strict saddle points} and shows that noisy stochastic gradient descent can escape saddle points for generic objectives $f$;
\cite{lee2016gradient} proves that gradient descent converges to (local) minimizers almost surely, using arguments drawn from dynamical systems theory.

More related to this paper are the works of \cite{ge2016matrix} and \cite{bhojanapalli2016global}: they show that matrix completion and sensing have no spurious local minima, for the case where $X^\star$ is square and PSD.
For both cases, extending these arguments for the more realistic non-square case is a non-trivial task. 

\subsection{Assumptions and Definitions}
We first state the assumptions we make for the matrix sensing setting.
We consider the case where the linear operator $\mathcal{A}$ satisfies the \emph{Restricted Isometry Property}, according to the following definition \cite{candes2011tight}:
\begin{definition}[Restricted Isometry Property (RIP)]\label{def:rip}
A linear operator $\mathcal{A} :~\R^{m \times n} \rightarrow \R^p$ satisfies the restricted isometry property on rank-$r$ matrices, with parameter $\delta_{r}$, if the following set of inequalities hold for all rank-$r$ matrices $X$:
\begin{align*}
(1 - \delta_r) \cdot \|X\|_F^2 \leq \|\mathcal{A}(X)\|_2^2 \leq (1 + \delta_r) \cdot \|X\|_F^2.
\end{align*}
\end{definition}
Characteristic examples are Gaussian-based linear maps \citep{fazel2008compressed, recht2010guaranteed}, 
Pauli-based measurement operators, used in quantum state tomography applications \citep{liu2011universal}, 
Fourier-based measurement operators, which lead to computational gains in practice due to their structure \citep{krahmer2011new, recht2010guaranteed}, 
or even permuted and sub-sampled noiselet linear operators, used in image and video compressive sensing applications \citep{waters2011sparcs}.

In this paper, we consider sensing mechanisms that can be expressed as:
$$
\left(\mathcal{A}(X)\right)_i = \left\langle A_i, X \right\rangle, \quad \forall i = 1, \dots, p, \text{  and  } A_i \in \R^{m \times n}.
$$
\emph{E.g.}, for the case of a Gaussian map $\mathcal{A}$, $A_i$ are independent, identically distributed (i.i.d.) Gaussian matrices; for the case of a Pauli map $\mathcal{A}$, $A_i \in \R^{n \times n}$ are i.i.d. and drawn uniformly at random from a set of scaled Pauli ``observables" $(P_1 \otimes P_2 \otimes \cdots \otimes P_d)/\sqrt{n}$, where $n = 2^d$ and $P_i$ is a $2 \times 2$ Pauli observable matrix \cite{liu2011universal}.

A useful property derived from the RIP definition is the following \cite{candes2008restricted}:
\begin{proposition}[Useful property due to RIP]\label{def:property_rip}
For a linear operator $\mathcal{A} :~\R^{m \times n} \rightarrow \R^p$ that satisfies the restricted isometry property on rank-$r$ matrices, the following inequality holds for any two rank-$r$ matrices $X, ~Y \in \R^{m \times n}$:
\begin{align*}
\left| \sum_{i=1}^p \ip{A_i}{X} \cdot \ip{A_i}{Y} - \ip{X}{Y} \right| \leq \delta_{2r} \cdot \normf{X} \cdot \normf{Y}.
\end{align*}
\end{proposition}

%
%

An important issue in optimizing $f$ over the factored space is the existence of non-unique possible factorizations for a given $\X$. 
Since we are interested in obtaining a low-rank solution in the original space, we need a notion of distance to the low-rank solution $\X^\star$ over the factors. 
Among infinitely many possible decompositions of $\X^\star$, we focus on the set of ``equally-footed'' factorizations \citep{tu2015low}:
\begin{align}
\mathcal{X}^\star_r = \Big\{ &\left(\Uo,\Vo\right):~\Uo \in \R^{m \times r}, \Vo \in \R^{n \times r}, \nonumber \\ &\quad \quad \quad \quad \quad \Uo {\Vo}^\top = X^\star, \sigma_i(\Uo) = \sigma_i(\Vo) = \sigma_i(X^\star)^{1/2}, \forall i \in [r]\Big\}. \label{prelim:eq_footing}
\end{align}

Given a pair $(\U,\V)$, we define the distance to $X^\star$ as:
\begin{align*}
\dist\left(\U,\V;X^\star \right) = \min_{(\Uo,\Vo) \in \mathcal{X}^\star_r} \normf{\begin{bmatrix} \U \\ \V \end{bmatrix} - \begin{bmatrix} \Uo \\ \Vo \end{bmatrix}}.
\end{align*}

\subsection{Problem Re-formulation}
Before we delve into the main results, we need to further reformulate the objective \eqref{eqn:formulation} for our analysis.
First, we use a well-known trick to reduce \eqref{eqn:formulation} to a semidefinite optimization. Let us define auxiliary variables
\begin{align*}
W = \begin{bmatrix} U \\ V \end{bmatrix} \in \R^{(m + n) \times r} &,\quad
\Wt = \begin{bmatrix} U \\ -V \end{bmatrix} \in \R^{(m + n) \times r}.
\end{align*}
Based on the auxiliary variables, we define the linear map $\mathcal{B}: ~\R^{(m+n) \times (m+n)} \rightarrow \R^p$ such that $(\mathcal{B}(WW^\top))_i = \langle B_i, WW^\top \rangle$, and $B_i \in \R^{(m+n) \times (m+n)}$. To make a connection between the variable spaces $(U, V)$ and $W$, $\mathcal{A}$ and $\mathcal{B}$ are related via matrices $A_i$ and $B_i$ as follows:
\begin{align*}
B_i = \frac{1}{2} \cdot \begin{bmatrix}
0 & A_i \\ 
A_i^\top & 0
\end{bmatrix}.
\end{align*}
This further implies that:
\begin{align*}
(\mathcal{B}(WW^\top))_i
&= \frac{1}{2} \cdot \langle B_i, WW^\top \rangle
 = \frac{1}{2} \cdot \left \langle \begin{bmatrix}
0 & A_i \\ 
A_i^\top & 0
\end{bmatrix},  \begin{bmatrix}
UU^\top & UV^\top \\
VU^\top & VV^\top
\end{bmatrix}
\right \rangle = \ip{A_i}{UV^\top}.
\end{align*}
Given the above, we re-define $f: \R^{(m+n) \times r} \rightarrow \R$ such that 
\begin{align}{\label{eq:new_problem}}
f(W) &:= \|\mathcal{B}(WW^\top) - b\|_2^2.
\end{align}
It is important to note that $\mathcal{B}$ operates on $(m + n) \times (m + n)$ matrices, while we assume RIP on $\mathcal{A}$ and $m \times n$ matrices.
Making no other assumptions for $\mathcal{B}$, we cannot directly apply \cite{bhojanapalli2016global} on \eqref{eq:new_problem}, but a rather different analysis is required.

In addition to this redefinition, we also introduce a \emph{regularizer} $g: \R^{(m+n) \times r} \rightarrow \R$ such that
$$
g(W) := \lambda \normf{\Wt^\top W}^2 = \lambda \normf{U^\top U - V^\top V}^2.
$$
This regularizer was first introduced in \cite{tu2015low} to prove convergence of its algorithm for non-square matrix sensing, and it is also used in this paper to analyze local minima of the problem.
After setting $\lambda = \frac{1}{4}$, \eqref{eqn:formulation} can be \emph{equivalently} written as:
\begin{equation} \label{eqn:formulation2}
\begin{aligned}
& \underset{W \in \R^{(m+n) \times r}}{\text{minimize}}
& & f(W) + g(W) := \|\mathcal{B}(WW^\top) - b\|_2^2 + \frac{1}{4} \cdot \normf{\Wt^\top W}^2.
\end{aligned}
\end{equation}
By equivalent, we note that the addition of $g$ in the objective does not change the problem, since for any rank-$r$ matrix $X$ there is a pair of factors $(U,V)$ such that $g(W) = 0$. 
It merely reduces the set of optimal points from all possible factorizations of $\Xs$ to \emph{balanced} factorizations of $\Xs$ in $\mathcal{X}^\star_r$. 
$\Us$ and $\Vs$ have the same set of singular values, which are the square roots of the singular values of $\Xs$. A key property of the balanced factorizations is the following.
\begin{proposition} \label{lem:truefactor_cancel}
For any factorization of the form \eqref{prelim:eq_footing}, it holds that
\begin{align*}
\Wt^{\star\top} \Ws = {\Us}^\top \Us - {\Vs}^\top \Vs = 0
\end{align*}
\end{proposition}
\begin{proof}
By ``balanced factorizations'' of $\Xs = \Us \Vs^\top$, we mean that factors $\Us$ and $\Vs$ satisfy
\begin{align} \label{eqn:fact_opt}
\Us = A \Sigma^{1/2} R ,\quad \Vs = B \Sigma^{1/2} R
\end{align}
where $\Xs = A \Sigma B^\top$ is the SVD, and $R \in \R^{r \times r}$ is an orthonormal matrix. 
Apply this to $\Wt^{\star\top} \Ws$ to get the result.
\end{proof}
\noindent Therefore, we have $g(\Ws) = 0$, and $(\Us,\Vs)$ is an optimal point of \eqref{eqn:formulation2}.

\section{Main Results}
This section describes our main results on the function landscape of the non-square matrix sensing problem.
The following theorem bounds the distance of any local minima to the global minimum, by the function value at the global minimum. 

\begin{theorem} \label{thm:mainresult}
Suppose $\Ws$ is any target matrix of the optimization problem \eqref{eqn:formulation2}, under the balanced singular values assumption for $\Us$ and $\Vs$.
If~$W$ is a critical point satisfying the first- and the second-order optimality conditions, i.e., $\nabla (f + g)(W) = 0$ and $\nabla^2(f + g)(W) \succeq 0$, 
then we have
\begin{align} \label{eqn:mainresult}
& \frac{1 - 5\delta_{2r} - 544 \delta_{4r}^2 - 1088 \delta_{2r} \delta_{4r}^2}{8(40+68\delta_{2r})(1+\delta_{2r})} \normf{WW^\top - W^\star W^{\star\top}}^2
\leq \norm{\mathcal{A}(\Us \Vs^{\top}) - b}^2.
\end{align}
\end{theorem}
Observe that for this bound to make sense, the term $\tfrac{1 - 5\delta_{2r} - 544 \delta_{4r}^2 - 1088 \delta_{2r} \delta_{4r}^2}{8(40+68\delta_{2r})(1+\delta_{2r})}$ needs to be positive. 
We provide some intuition of this result next.
Combined with Lemma 5.14 in \cite{tu2015low}, we can also obtain the distance between $(U,V)$ and $(\Us,\Vs)$.
\begin{corollary}\label{cor:mainresult}
For $W = \begin{bmatrix} U \\ V \end{bmatrix}$ and given the assumptions of Theorem \ref{thm:mainresult}, we have
\begin{align} \label{eqn:mainresult2}
& \sigma_r(\Xs) \cdot \frac{1-5\delta_{2r}-544 \delta_{4r}^2-1088 \delta_{2r}\delta_{4r}^2}{10(40+68\delta_{2r})(1+\delta_{2r})} \cdot 
\dist\left(\U,\V;X^\star \right)^2
\leq \norm{\mathcal{A}(\Us \Vs^{\top}) - b}^2.
\end{align} 
\end{corollary}

Implications of these results are described next, where we consider specific settings.

\begin{remark}[Noiseless matrix sensing]
Suppose that $\Ws = \begin{bmatrix} \Us \\ \Vs \end{bmatrix}$ is the underlying unknown true matrix, \emph{i.e.}, $\X^\star = \Us \Vs^\top$ is rank-$r$ and $b = \mathcal{A}(\Us \Vs^\top)$. 
We assume the noiseless setting, $w = 0$. 
If $0 \leq \delta_{2r} \le \delta_{4r} \lesssim 0.0363$, then $\tfrac{1 - 5\delta_{2r} - 544 \delta_{4r}^2 - 1088 \delta_{2r} \delta_{4r}^2}{10(40+68\delta_{2r})(1+\delta_{2r})} > 0$ in Corollary \ref{cor:mainresult}.
Since the RHS of \eqref{eqn:mainresult2} is zero, this further implies that $\dist\left(\U,\V;X^\star \right) = 0$, \emph{i.e.}, any critical point $W$ that satisfies first- and second-order optimality conditions is global minimum.
\end{remark}

\begin{remark}[Noisy matrix sensing]
Suppose that $\Ws$ is the underlying true matrix, such that $\X^\star = \Us\Vs^\top$ and is rank-$r$, and $b = \mathcal{A}(\Us \Vs^\top) + w$, for some noise term $w$. 
If $0 \le \delta_{2r} \le \delta_{4r} < 0.02$, then it follows from \eqref{eqn:mainresult} that for any local minima $W$ the distance to $\Us\Vs^\top$ is bounded by
$$
\frac{1}{500} \normf{WW^\top - W^\star W^{\star\top}} 
\leq \norm{w}.
$$
\end{remark}

\begin{remark}[High-rank matrix sensing]
Suppose that $\Xs$ is of arbitrary rank and let $\X^\star_r$ denote its best rank-$r$ approximation. 
Let $b = \mathcal{A}(\Xs) + w$ where $w$ is some noise and let $(\Us,\Vs)$ be a balanced factorization of $X_r^\star$. 
If $0 \le \delta_{2r} \le \delta_{4r} < 0.005$, then it follows from \eqref{eqn:mainresult2} that for any local minima $(U, V)$ the distance to $(\Us, \Vs)$ is bounded by
$$ 
\dist\left(\U,\V;X^\star \right) \leq \tfrac{1250}{3\sigma_r(X^\star)} \cdot \norm{\mathcal{A}(\Xs-X_r^\star) + w}.
$$
\end{remark}


\section{Proof of Main Results}
We first describe the first- and second-order optimality conditions for $f + g$ objective with $W$ variable.
Then, we provide a detailed proof of the main results: by carefully analyzing the conditions, we study how a local optimum is related to the global optimum.

\subsection{Gradient and Hessian of $f$ and $g$}
The gradients of $f$ and $g$ w.r.t. $W$ are given by:
\begin{align*}
\nabla f(W) &= \sum_{i = 1}^p (\ip{B_i}{WW^\top} - b_i) \cdot B_i \cdot W \\
\nabla g(W) &= \frac{1}{4} \Wt {\Wt}^\top W \quad \left(\equiv \frac{1}{4} \cdot \begin{bmatrix}U\\-V\end{bmatrix} \cdot (U^\top U - V^\top V) \right)
\end{align*}
Regarding Hessian information, we are interested in the positive semi-definiteness of $\nabla^2 (f+g)$; for this case, it is easier to write the second-order Hessian information with respect to to some matrix direction $Z \in \R^{(m+n) \times r}$, as follows:
\begin{align*}
\texttt{vec}(Z)^\top &\cdot \nabla^2 f(W) \cdot \texttt{vec}(Z) \nonumber \\ 
&= \ip{\lim_{t \rightarrow 0}  \left[ \tfrac{\nabla f(W + tZ) - \nabla f(W)}{t} \right]}{Z} \nonumber \\
&= \ip{\sum_{i = 1}^p \langle B_i, ZW^\top + WZ^\top \rangle \cdot B_i W}{Z} + \ip{\sum_{i = 1}^p \left(\ip{B_i}{WW^\top} - b_i\right) \cdot B_i Z}{Z} \\
&= \sum_{i = 1}^p \langle B_i, ZW^\top + WZ^\top \rangle \cdot \ip{B_i}{ZW^\top} + \sum_{i = 1}^p \left(\ip{B_i}{WW^\top} - b_i\right) \cdot \langle B_i, ZZ^\top \rangle \\ \\
\texttt{vec}(Z)^\top &\cdot \nabla^2 g(W) \cdot \texttt{vec}(Z) \nonumber \\
&= \ip{\lim_{t \rightarrow 0}  \left[ \tfrac{\nabla g(W + tZ) - \nabla g(W)}{t} \right]}{Z} \\
&= \frac{1}{4} \ip{\tilde{Z} {\Wt}^\top W}{Z} + \frac{1}{4} \ip{\Wt \tilde{Z}^\top W}{Z} + \frac{1}{4} \ip{\Wt {\Wt}^\top Z}{Z}\\
&= \frac{1}{4} \ip{\tilde{Z} {\Wt}^\top}{ZW^\top} + \frac{1}{4} \ip{\Wt \tilde{Z}^\top}{ZW^\top} + \frac{1}{4} \ip{\Wt {\Wt}^\top}{ZZ^\top}.
\end{align*}

\subsection{Optimality conditions}
Given the expressions above, we now describe first- and second-order optimality conditions on the composite objective $f + g$.
\paragraph{First-order optimality condition.}
By the first-order optimality condition of a pair $(U, ~V)$ such that $W = \begin{bmatrix} U \\ V \end{bmatrix}$, we have $\nabla (f+g) \left( W \right) = 0$. 
This further implies:
\begin{align} \label{eqn:firstorderopt}
&\nabla (f+g) \left( W \right) = 0 \quad 
\Rightarrow \quad \sum_{i = 1}^p \left(\ip{B_i}{WW^\top} - b_i\right) \cdot B_i \cdot W + \frac{1}{4} \Wt \Wt^\top W = 0
\end{align}

\paragraph{Second-order optimality condition.}
For a point $W$ that satisfies the second-order optimality condition $\nabla^2 (f + g)(W) \succeq 0$, the following holds for any $Z \in \R^{(m+n) \times r}$:
\begin{align} \label{eqn:secondorderopt}
\texttt{vec}(Z)^\top &\cdot \nabla^2 (f+g)(W) \cdot \texttt{vec}(Z) \geq 0 \quad \nonumber \\
&\Rightarrow \quad \sum_{i = 1}^p \langle B_i, ZW^\top + WZ^\top \rangle \cdot \ip{B_i}{ZW^\top} + \sum_{i=1}^p \left(\ip{B_i}{WW^\top} - b_i\right) \cdot \langle B_i, ZZ^\top \rangle \nonumber \\
&\qquad \qquad \qquad + \frac{1}{4} \ip{\Zt \Wt^\top + \Wt \Zt^\top}{Z W^\top} + \frac{1}{4} \ip{\Wt {\Wt}^\top}{ZZ^\top} \geq 0
\end{align}

\subsection{Proof of Theorem \ref{thm:mainresult}}
Suppose that $W$ is a critical point satisfying the optimality conditions \eqref{eqn:firstorderopt} and \eqref{eqn:secondorderopt}. The second order optimality is again written as
\begin{small}
\begin{align} \label{eqn:secondorderopt2} 
& \underbrace{\sum_{i = 1}^p \langle B_i, Z W^\top + WZ^\top\rangle \cdot \ip{B_i}{ZW^\top}}_{(A)}
 + \underbrace{\sum_{i = 1}^p \left(\ip{B_i}{WW^\top} - b_i\right) \cdot \ip{B_i}{ZZ^\top}}_{(B)} \nonumber \\
&\qquad \qquad \qquad + \frac{1}{4} \underbrace{\ip{\Zt {\Wt}^\top}{Z W^\top}}_{(C)}
 + \frac{1}{4} \underbrace{\ip{\Wt \Zt^\top}{Z W^\top}}_{(D)}
 + \frac{1}{4} \underbrace{\ip{\Wt {\Wt}^\top}{ZZ^\top}}_{(E)} \geq 0,\qquad \forall Z = \begin{bmatrix} Z_U \\ Z_V \end{bmatrix} \in \R^{(m+n) \times r}.
\end{align}
\end{small}
As in \cite{bhojanapalli2016global}, we sum up the above condition for $Z_1 \triangleq (W - \Ws R) e_1 e_1^\top, \ldots, Z_r \triangleq (W - \Ws R) e_r e_r^\top$. For simplicity, we first assume $Z = W - \Ws R$.

\paragraph{Bounding terms (A), (C) and (D).} The following bounds work for any $Z$.
\begin{align*}
(A)
&= \sum_{i=1}^p \ip{B_i}{ZW^\top}^2 + \sum_{i = 1}^p \ip{B_i}{ZW^\top} \cdot \ip{B_i}{WZ^\top} \\
&\stackrel{(a)}{=} 2 \cdot \sum_{i=1}^p \ip{B_i}{ZW^\top}^2 \\
&= \frac{1}{2} \sum_{i=1}^p \left(\ip{A_i}{Z_U V^\top} + \ip{A_i}{U Z_V^\top}\right)^2 \\
&\stackrel{(b)}{\le} \frac{1+\delta_{2r}}{2} \normf{Z_U V^\top}^2 + \frac{1+\delta_{2r}}{2} \normf{U Z_V^\top}^2
   + \ip{Z_U V^\top}{U Z_V^\top} + \delta_{2r} \cdot \normf{Z_U V^\top} \cdot \normf{U Z_V^\top}\\
&\stackrel{(c)}{\le} \underbrace{\frac{1+2\delta_{2r}}{2} \normf{Z_U V^\top}^2 + \frac{1+2\delta_{2r}}{2} \normf{U Z_V^\top}^2}_{(A1)}
   + \underbrace{\ip{Z_U V^\top}{U Z_V^\top}}_{(A2)}
\end{align*}
where (a) follows from that every $B_i$ is symmetric, (b) follows from Proposition \ref{def:property_rip}, and (c) follows from the AM-GM inequality. We also have
\begin{small}
\begin{align*}
(C)
&= \ip{\tilde{Z} {\Wt}^\top}{Z W^\top} = \normf{Z_U U^\top}^2 + \normf{Z_V V^\top}^2 - \normf{Z_U V^\top}^2 - \normf{Z_V U^\top}^2, \\ 
(A1) + \frac{1}{4} (C)
&\le \frac{1+4\delta_{2r}}{4} \normf{ZW^\top}^2, \\ 
(D)
&= \ip{\Wt \tilde{Z}^\top}{Z W^\top} = \ip{U Z_U^\top}{Z_U U^\top} + \ip{V Z_V^\top}{Z_V V^\top} - \ip{U Z_V^\top}{Z_U V^\top} - \ip{V Z_U^\top}{Z_V U^\top}, \\ 
(A2) + \frac{1}{4} (D)
&= \frac{1}{4} \ip{WZ^\top}{ZW^\top}, \\ 
(A) + \frac{1}{4} (C) + \frac{1}{4} (D)
&\le \frac{1}{8} \normf{WZ^\top + ZW^\top}^2 + \delta_{2r} \normf{ZW^\top}^2.
\end{align*}
\end{small}

\paragraph{Bounding terms (B) and (E).} We have
\begin{small}
\begin{align*}
(B)
&= \sum_{i = 1}^p \left(\ip{B_i}{WW^\top} - y_i\right) \cdot \ip{B_i}{Z Z^\top} = \sum_{i = 1}^p \left(\ip{B_i}{WW^\top} - y_i\right) \cdot \ip{B_i}{(W-\Ws R)(W-\Ws R)^\top} \\
&= \ip{\sum_{i = 1}^p \left(\ip{B_i}{WW^\top} - y_i\right) \cdot B_i}{WW^\top + \Ws\Ws^\top - 2 \Ws RW^\top} 
\end{align*}
\begin{align*}
&\stackrel{(a)}{=} - \sum_{i = 1}^p \left(\ip{B_i}{WW^\top} - y_i\right) \cdot \ip{B_i}{WW^\top - \Ws {\Ws}^\top}
   - \frac{1}{2} \ip{\Wt \Wt^\top}{(W - \Ws R)W^\top} \\
&= - \sum_{i = 1}^p \ip{B_i}{WW^\top}^2
   - \frac{1}{2} \ip{\Wt \Wt^\top}{(W - \Ws R)W^\top} - \sum_{i=1}^p \left(\ip{B_i}{\Ws {\Ws}^\top} - y_i\right) \cdot \ip{B_i}{WW^\top - \Ws {\Ws}^\top} \\
&\stackrel{(b)}{\le} - \underbrace{(1-\delta_{2r}) \normf{UV^\top - \Us {\Vs}^\top}^2}_{(B1)}
     - \frac{1}{4} \cdot \underbrace{ \ip{\Wt \Wt^\top}{2ZW^\top}}_{(B2)}
     - \underbrace{\sum_{i=1}^p \left(\ip{B_i}{\Ws {\Ws}^\top} - y_i\right) \cdot \ip{B_i}{WW^\top - \Ws {\Ws}^\top}}_{(B3)}
\end{align*}
\end{small}
where at (a) we add the first-order optimality equation
\begin{align*}
\ip{\sum_{i = 1}^p \left(\ip{B_i}{WW^\top} - y_i\right) \cdot B_i \cdot W}{2W - 2\Ws R} = - \frac{1}{2} \ip{\Wt \Wt^\top W}{W - \Ws R},
\end{align*}
and (b) follows from Proposition \ref{def:property_rip}. Then we have
\begin{align*}
(B2)-(E) &= \ip{\Wt \Wt^\top}{2ZW^\top - ZZ^\top}\\
&\stackrel{(a)}{=} \ip{\Wt \Wt^\top}{2WW^\top - W^\star R W^\top - W R^\top W^{\star\top} - (W-W^\star R) (W-W^\star R)^\top} \\
&= \ip{\Wt \Wt^\top}{WW^\top - W^\star W^{\star\top}} \\
&\stackrel{(b)}{=} \ip{\Wt \Wt^\top}{WW^\top - W^\star W^{\star\top}} + \ip{\Wt^\star \Wt^{\star\top}}{W^\star W^{\star\top}} \\
&\stackrel{(c)}{\ge} \ip{\Wt \Wt^\top}{WW^\top - W^\star W^{\star\top}} + \ip{\Wt^\star \Wt^{\star\top}}{W^\star W^{\star\top}} - \ip{\Wt^\star \Wt^{\star\top}}{W W^\top} \\
&= \ip{\Wt \Wt^\top - \Wt^\star \Wt^{\star\top}}{WW^\top - W^\star W^{\star\top}}
\end{align*}
where (a) follows from that $\Wt \Wt^\top$ is symmetric, (b) follows from Proposition \ref{lem:truefactor_cancel}, (c) follows from that the inner product of two PSD matrices is non-negative. We then have,
\begin{small}
\begin{align*}
(B1)+\frac{1}{4}(B2)-\frac{1}{4}(E) &\ge (1-\delta_{2r}) \normf{UV^\top - U^\star V^{\star\top}}^2 + \frac{1}{4} \ip{\Wt \Wt^\top - \Wt^\star \Wt^{\star\top}}{WW^\top - W^\star W^{\star\top}} \\
&= \left(1-\delta_{2r} - \frac{1}{2}\right) \normf{UV^\top - U^\star V^{\star\top}}^2 + \frac{1}{4} \normf{UU^\top - U^\star U^{\star\top}}^2 + \frac{1}{4} \normf{VV^\top - V^\star V^{\star\top}}^2 \\
&\ge \frac{1-2\delta_{2r}}{4} \cdot \normf{WW^\top - W^\star W^{\star\top}}^2
\end{align*}
\end{small}
For (B3), we have
\begin{align*}
-(B3)
&= \sum_{i=1}^p \left(\ip{B_i}{\Ws {\Ws}^\top} - b_i\right) \cdot \ip{B_i}{WW^\top - \Ws {\Ws}^\top} \\
&\stackrel{(a)}{\le} \norm{\mathcal{A}(\Us \Vs^{\top}) - b}
\cdot \left( \sum_{i=1}^p \ip{B_i}{WW^\top - \Ws {\Ws}^\top}^2 \right)^{\frac{1}{2}} \\
&\stackrel{(b)}{\le} \sqrt{1+\delta_{2r}} \cdot \norm{\mathcal{A}(\Us \Vs^{\top}) - b}
\cdot \normf{WW^\top - \Ws {\Ws}^\top}
\end{align*}
where (a) follows from the Cauchy-Schwarz inequality, and (b) follows from Proposition \ref{def:property_rip}. We finally get
\begin{small}
\begin{align} \label{eqn:bound_be}
(B) + \frac{1}{4} (E)
&\le - \frac{1-2\delta_{2r}}{4} \cdot \normf{WW^\top - W^\star W^{\star\top}}^2 + \sqrt{1+\delta_{2r}} \cdot \norm{\mathcal{A}(\Us \Vs^{\top}) - b} \cdot \normf{WW^\top - \Ws {\Ws}^\top} \nonumber \\
&\le - \frac{3-8\delta_{2r}}{16} \cdot \normf{WW^\top - W^\star W^{\star\top}}^2 + 16(1+\delta_{2r}) \cdot \norm{\mathcal{A}(\Us \Vs^{\top}) - b}^2
\end{align}
\end{small}
where the last inequality follows from the AM-GM inequality.

\paragraph{Summing up the inequalities for $Z_1,\ldots,Z_r$.}
Now we apply $Z_j = Z e_j e_j^\top$. Since $ZZ^\top = \sum_{j=1}^r Z_j Z_j^\top$ in \eqref{eqn:secondorderopt2}, the analysis does not change for (B) and (E). For (A), (C), and (D), we obtain
\begin{align*}
(A) + \frac{1}{4} (C) + \frac{1}{4} (D)
&\le \sum_{j=1}^r \left\{ \frac{1}{8} \normf{WZ_j^\top + Z_jW^\top}^2 + \delta_{2r} \normf{Z_jW^\top}^2 \right\}
\end{align*}
We have
\begin{align*}
\sum_{j=1}^r \normf{WZ_j^\top + Z_jW^\top}^2
&= \sum_{j=1}^r \normf{W e_j e_j^\top Z^\top + Z e_j e_j^\top W^\top}^2 \\
&= 2 \cdot \sum_{j=1}^r \normf{W e_j e_j^\top Z^\top}^2 + 2 \cdot \sum_{i=1}^r \ip{W e_j e_j^\top Z^\top}{Z e_j e_j^\top W^\top}\\
&= 2 \cdot \sum_{j=1}^r \normf{W e_j e_j^\top Z^\top}^2 + 2 \cdot \sum_{i=1}^r (e_j^\top Z^\top W e_j)^2\\
&\le 2 \cdot \sum_{j=1}^r \normf{W e_j e_j^\top Z^\top}^2 + 2 \cdot \sum_{i=1}^r \norm{Z e_j}^2 \cdot \norm{W e_j}^2\\
&= 4 \cdot \sum_{j=1}^r \normf{W e_j e_j^\top Z^\top}^2
\end{align*}
where the inequality follows from the Cauchy-Schwarz inequality. Applying this bound, we get
\begin{align}{\label{eq:ACD}}
(A) + \frac{1}{4} (C) + \frac{1}{4} (D) \le \frac{1 + 2\delta_{2r}}{2} \sum_{j=1}^r \normf{W e_j e_j^\top (W-\Ws R)}^2.
\end{align}
Next, we re-state \cite[Lemma 4.4]{bhojanapalli2016global}:
\begin{lemma}
Let $W$ and $W^\star$ be two matrices, and $Q$ is an orthonormal matrix that spans the column space of $W$.
Then, there exists an orthonormal matrix $R$ such that, for any stationary point $W$ of $g(W)$ that satisfies first and second order condition, the following holds:
\begin{align}{\label{eq:ACD2}}
\sum_{j = 1}^r \|W e_j e_j^\top (W - W^\star R)\|_F^2 \leq \tfrac{1}{8} \cdot \|WW^\top - W^\star W^{\star \top}\|_F^2 + \tfrac{34}{8} \cdot \|(WW^\top - W^\star W^{\star \top})QQ^\top\|_F^2
\end{align}
\end{lemma}
And we have the following variant of \cite[Lemma 4.2]{bhojanapalli2016global}.
\begin{lemma}\label{lem:firstorder}
For any pair of points $(U, ~V)$ that satisfies the first-order optimality condition, and $\mathcal{A}$ be a linear operator satisfying the RIP condition with parameter $\delta_{4r}$, the following inequality holds:
\begin{align}{\label{eq:ACD3}}
\frac{1}{4} \cdot \normf{(WW^\top - \Ws{\Ws}^\top)QQ^\top}
&\le \delta_{4r} \cdot \normf{WW^\top - \Ws {\Ws}^\top} + \sqrt{\frac{1+\delta_{2r}}{2}} \cdot \norm{\mathcal{A}(\Us\Vs^\top) - b}
\end{align}
\end{lemma}
Applying the above two lemmas, we can get
\begin{small}
\begin{align} \label{eqn:bound_acd}
(A) + \frac{1}{4} (C) + \frac{1}{4} (D) \le \frac{(1 + 2\delta_{2r}) \cdot (1 + 1088 \delta_{4r}^2)}{16} \normf{WW^\top - \Ws\Ws^\top}^2 + 34(1+2\delta_{2r})(1+\delta_{2r}) \norm{\mathcal{A}(\Us\Vs^\top)-b}^2.
\end{align}
\end{small}

\paragraph{Final inequality.} Plugging \eqref{eqn:bound_acd} and \eqref{eqn:bound_be} to \eqref{eqn:secondorderopt2}, we get
\begin{small}
\begin{align*}
& \frac{-1 + 5\delta_{2r} + 544\delta_{4r}^2 + 1088 \delta_{2r} \delta_{4r}^2}{8} \normf{WW^\top - W^\star W^{\star\top}}^2 
+ (40 + 68\delta_{2r})(1+\delta_{2r}) \cdot \norm{\mathcal{A}(\Us \Vs^{\top}) - b}^2 \geq 0.
\end{align*}
\end{small}
Finally we have
\begin{align*}
& \frac{1 - 5\delta_{2r} - 544 \delta_{4r}^2 - 1088 \delta_{2r} \delta_{4r}^2}{8(40+68\delta_{2r})(1+\delta_{2r})} \normf{WW^\top - W^\star W^{\star\top}}^2 
\leq \norm{\mathcal{A}(\Us \Vs^{\top}) - b}^2,
\end{align*}
which completes the proof.

\subsubsection{Proof of Lemma \ref{lem:firstorder}}
The first-order optimality condition can be written as
\begin{small}
\begin{align*}
0 &= \ip{\nabla (f+g)(W)}{Z} \\
&= \sum_{i=1}^p \left(\ip{B_i}{WW^\top} - b_i\right) \cdot \ip{B_i W}{Z} + \frac{1}{4} \ip{\Wt \Wt^\top W}{Z} \\
&= \sum_{i=1}^p \ip{B_i}{WW^\top - \Ws {\Ws}^\top} \ip{B_i}{ZW^\top} + \sum_{i=1}^p \left(\ip{B_i}{\Ws\Ws^\top} - b_i \right) \cdot \ip{B_i}{ZW^\top} + \frac{1}{4} \ip{\Wt \Wt^\top}{ZW^\top} \\
&= \frac{1}{2} \cdot \sum_{i=1}^p \ip{A_i}{UV^\top - \Us {\Vs}^\top} \ip{A_i}{Z_U V^\top + U Z_V^\top} \\
&\qquad + \frac{1}{2} \cdot \sum_{i=1}^p \left(\ip{A_i}{\Us\Vs^\top} - b_i \right) \cdot \ip{A_i}{Z_U V^\top + U Z_V^\top} + \frac{1}{4} \ip{\Wt \Wt^\top}{ZW^\top}
,\qquad \qquad \forall Z = \begin{bmatrix} Z_U \\ Z_V \end{bmatrix} \in \R^{(m+n) \times r}.
\end{align*}
\end{small}
Applying Proposition \ref{def:property_rip} and the Cauchy-Schwarz inequality to the condition, we obtain
\begin{small}
\begin{align} \label{eqn:firstoptlemma}
&\frac{1}{2} \cdot \underbrace{\ip{UV^\top - \Us {\Vs}^\top}{Z_U V^\top + U Z_V^\top}}_{(A)} 
+ \frac{1}{4} \cdot \underbrace{\ip{\Wt \Wt^\top}{ZW^\top}}_{(B)} \nonumber \\
&\le \delta_{4r} \cdot \underbrace{\normf{UV^\top - \Us {\Vs}^\top} \cdot \normf{Z_U V^\top + U Z_V^\top}}_{(C)}
+ \frac{\sqrt{1+\delta_{2r}}}{2} \cdot \underbrace{\norm{\mathcal{A}(\Us\Vs^\top) - b} \cdot \normf{Z_U V^\top + U Z_V^\top}}_{(D)}
\end{align}
\end{small}
Let $Z = (WW^\top - \Ws {\Ws}^\top) Q R^{-1 \top}$ where $W = QR$ is the QR decomposition. Then we obtain
$$
ZW^\top = (WW^\top - \Ws {\Ws}^\top) QQ^\top.
$$
We have
\begin{align*}
2 (A)
&= 2 \ip{\begin{bmatrix} 0 & UV^\top - \Us{\Vs}^\top \\ VU^\top - \Vs{\Us}^\top & 0 \end{bmatrix}}{ZW^\top} \\
&= \ip{(WW^\top - \Wt{\Wt}^\top) - (\Ws{\Ws}^\top - {\Wt}^\star {\Wt}^{\star\top})}{(WW^\top - \Ws{\Ws}^\top)QQ^\top}, \\ \\
(B)
&= \ip{\Wt \Wt^\top}{(WW^\top - \Ws {\Ws}^\top) QQ^\top} \\
&\stackrel{(a)}{=} \ip{\Wt \Wt^\top}{(WW^\top - \Ws {\Ws}^\top) QQ^\top} + \ip{\Wt^\star \Wt^{\star\top}}{\Ws {\Ws}^\top QQ^\top} \\
&\stackrel{(b)}{\ge} \ip{\Wt \Wt^\top}{(WW^\top - \Ws {\Ws}^\top) QQ^\top} - \ip{\Wt^\star \Wt^{\star\top}}{(WW^\top - \Ws {\Ws}^\top) QQ^\top} \\
&= \ip{\Wt \Wt^\top - \Wt^\star \Wt^{\star\top}}{(WW^\top - \Ws {\Ws}^\top) QQ^\top} \\
\end{align*}
where (a) follows from Proposition \ref{lem:truefactor_cancel}, and (b) follows from that the inner product of two PSD matrices is non-negative. Then we obtain
\begin{align*}
2(A) + (B)
&\ge \ip{WW^\top - \Ws{\Ws}^\top}{(WW^\top - \Ws{\Ws}^\top)QQ^\top} \\
&= \normf{(WW^\top - \Ws{\Ws}^\top)Q}^2 \\
&= \normf{(WW^\top - \Ws{\Ws}^\top)QQ^\top}^2
\end{align*}
For (C), we have
\begin{align*}
(C)
&= \normf{UV^\top - \Us {\Vs}^\top} \cdot \normf{Z_U V^\top + U Z_V^\top} \\
&\le \frac{1}{\sqrt{2}} \cdot \normf{WW^\top - \Ws {\Ws}^\top} \cdot \sqrt{ 2 \normf{Z_U V^\top}^2 + 2 \normf{U Z_V^\top}^2 } \\
&\le \normf{WW^\top - \Ws {\Ws}^\top} \cdot \sqrt{ \normf{Z W^\top}^2 } \\
&= \normf{WW^\top - \Ws {\Ws}^\top} \cdot \normf{(WW - \Ws{\Ws}^\top)QQ^\top}
\end{align*}
Plugging the above bounds into \eqref{eqn:firstoptlemma}, we get
\begin{small}
\begin{align*}
\frac{1}{4} \cdot \normf{(WW^\top - \Ws{\Ws}^\top)QQ^\top}^2
&\le \delta_{4r} \cdot \normf{WW^\top - \Ws {\Ws}^\top} \cdot \normf{(WW^\top - \Ws{\Ws}^\top)QQ^\top} \\
&\qquad + \sqrt{\frac{1+\delta_{2r}}{2}} \cdot \norm{\mathcal{A}(\Us\Vs^\top) - b} \cdot \normf{(WW^\top - \Ws{\Ws}^\top)QQ^\top}
\end{align*}
\end{small}
In either case of $\normf{(WW^\top - \Ws{\Ws}^\top)QQ^\top}$ being zero or positive, we can obtain
\begin{small}
\begin{align*}
\frac{1}{4} \cdot \normf{(WW^\top - \Ws{\Ws}^\top)QQ^\top}
&\le \delta_{4r} \cdot \normf{WW^\top - \Ws {\Ws}^\top} + \sqrt{\frac{1+\delta_{2r}}{2}} \cdot \norm{\mathcal{A}(\Us\Vs^\top) - b}
\end{align*}
\end{small}
This completes the proof.

\section{What About Saddle Points?}

Our discussion so far concentrates on whether $UV^\top$ parametrization introduces spurious local minima. 
Our main results show that any point $(U, V)$ that satisfies both first- and second-order optimality conditions\footnote{Note here that the second-order optimality condition includes positive \emph{semi}-definite second-order information; \emph{i.e.}, Theorem \ref{thm:mainresult} also handles saddle points due to the semi-definiteness of the Hessian at these points.} should be (or lie close to) the global optimum.
However, we have not discussed what happens with saddle points, \emph{i.e.}, points $(U,V)$ where the Hessian matrix contains both positive and negative eigenvalues.\footnote{Here, we do not consider the harder case where saddle points have Hessian with positive, negative and zero eigenvalues.}
This is important for practical reasons: first-order methods rely on gradient information and, thus, can easily get stuck to saddle points that may be far away from the global optimum.

\cite{ge2015escaping} studied conditions of the objective that guarantee that stochastic gradient descent---randomly initialized---converges to a local minimum; \emph{i.e.}, we can avoid getting stuck to non-degenerate saddle points. 
These conditions include $f + g$ being bounded and smooth, having Lipschitz Hessian, being locally strongly convex, and satisfying the strict saddle property, according to the following definition.

\begin{definition} \cite{ge2015escaping}
A twice differentiable function $f + g$ is strict saddle, if all its stationary points, that are not local minima, satisfy $\lambda_{\min}(\nabla^2 (f + g)(W)) < 0$.
\end{definition}

\cite{lee2016gradient} relax some of these conditions and prove the following theorem (for standard gradient descent).
\begin{theorem}[\cite{lee2016gradient} - Informal]
If the objective is twice differentiable and satisfies the strict saddle property, then gradient descent, randomly initialized and with sufficiently small step size, converges to a local minimum almost surely.
\end{theorem}

In this section, based on the analysis in \cite{bhojanapalli2016global}, we show that $f + g$ satisfy the strict saddle property, which implies that gradient descent can avoid saddle points and converge to the global minimum, with high probability.

\begin{theorem}\label{thm:saddle}
Consider noiseless measurements $b = \mathcal{A}(X^\star)$, with $\mathcal{A}$ satisfying RIP with constant $\delta_{4r} \leq \tfrac{1}{100}$. 
Assume that $\text{rank}(X^\star) = r$.
Let $(U, V)$ be a pair of factors that satisfies the first order optimality condition $\nabla f(W) = 0$, for $W = \begin{bmatrix}
U \\ V
\end{bmatrix}$, and $UV^\top \neq X^\star$. 
Then, 
$$
\lambda_{\min}\left( \nabla^2 (f + g)(W)\right) \leq -\frac{1}{7} \cdot \sigma_r(X^\star).
$$
\end{theorem}

\begin{proof}
Let $Z \in \mathbb{R}^{(m+n) \times r}$. 
Then, by \eqref{eqn:secondorderopt}, the proof of Theorem \ref{thm:mainresult} and the fact that $b = \mathcal{A}(X^\star)$ (noiseless), $\nabla^2 (f+g)(W)$ satisfies the following:
\begin{small}
\begin{align} \label{eqn:secondorderopt3} 
\texttt{vec}(Z)^\top \cdot \nabla^2 (f+g)(W) \cdot \texttt{vec}(Z) &= \sum_{i = 1}^p \langle B_i, Z W^\top + WZ^\top\rangle \cdot \ip{B_i}{ZW^\top}
 + \sum_{i = 1}^p \left(\ip{B_i}{WW^\top} - b_i\right) \cdot \ip{B_i}{ZZ^\top} \nonumber \\
&\qquad \qquad \qquad + \frac{1}{4} \ip{\Zt {\Wt}^\top}{Z W^\top}
 + \frac{1}{4} \ip{\Wt \Zt^\top}{Z W^\top}
 + \frac{1}{4} \ip{\Wt {\Wt}^\top}{ZZ^\top} \nonumber \\ 
 &\stackrel{\eqref{eq:ACD}, \eqref{eqn:bound_be}}{\leq} \frac{1 + 2\delta_{2r}}{2} \sum_{j=1}^r \normf{W e_j e_j^\top (W-\Ws R)}^2 - \frac{3-8\delta_{2r}}{16} \cdot \normf{WW^\top - W^\star W^{\star\top}}^2 \nonumber \\
 &\stackrel{\eqref{eq:ACD2}, \eqref{eq:ACD3}}{\leq} \left(\tfrac{1 + 2\delta_{2r}}{16} \cdot (1 + 34 \cdot 16\delta_{4r}^2) - \tfrac{3 - 8\delta_{2r}}{16}\right) \cdot \normf{WW^\top - W^\star W^{\star\top}}^2 \nonumber \\
 &\leq \tfrac{-1 + 5\delta_{4r} + 272 \delta_{4r}^2 + 544\delta_{4r}^3}{8} \cdot \normf{WW^\top - W^\star W^{\star\top}}^2 \nonumber \\
 &\leq -\tfrac{1}{10} \cdot \normf{WW^\top - W^\star W^{\star\top}}
\end{align}
\end{small} where the last inequality is due to the requirement $\delta_{4r} \leq \tfrac{1}{100}$.
For the LHS of \eqref{eqn:secondorderopt3}, we can lower bound as follows:
\begin{align*}
\texttt{vec}(Z)^\top \cdot \nabla^2 (f+g)(W) \cdot \texttt{vec}(Z) &\geq \|Z\|_F^2 \cdot \lambda_{\min}\left( \nabla^2 (f + g)(W)\right) \\
&= \|W - W^\star R\|_F^2 \cdot \lambda_{\min}\left( \nabla^2 (f + g)(W)\right)
\end{align*} where the last equality is by setting $Z = W - W^\star R$.
Combining this expression with \eqref{eqn:secondorderopt3}, we obtain:
\begin{align*}
\lambda_{\min}\left( \nabla^2 (f + g)(W)\right) &\leq - \frac{\sfrac{1}{10}}{\|W - W^\star R\|_F^2} \cdot \normf{WW^\top - W^\star W^{\star\top}} \\
&\stackrel{\text{Lemma 5.4, \cite{tu2015low}}}{\leq} - \frac{\sfrac{1}{10}}{\|W - W^\star R\|_F^2} \cdot 2(\sqrt{2}-1) \cdot \sigma_r(X^\star) \cdot \|W - W^\star R\|_F^2 \\
&\leq -\frac{1}{7} \cdot \sigma_r(X^\star).
\end{align*}
This completes the proof.
\end{proof}

\bibliography{NonSquareFGD}

\begin{thebibliography}{45}
\providecommand{\natexlab}[1]{#1}
\providecommand{\url}[1]{\texttt{#1}}
\expandafter\ifx\csname urlstyle\endcsname\relax
  \providecommand{\doi}[1]{doi: #1}\else
  \providecommand{\doi}{doi: \begingroup \urlstyle{rm}\Url}\fi

\bibitem[Aaronson(2007)]{aaronson2007learnability}
S.~Aaronson.
\newblock The learnability of quantum states.
\newblock In \emph{Proceedings of the Royal Society of London A: Mathematical,
  Physical and Engineering Sciences}, volume 463, pages 3089--3114. The Royal
  Society, 2007.

\bibitem[Anandkumar and Ge(2016)]{anandkumar2016efficient}
A.~Anandkumar and R.~Ge.
\newblock Efficient approaches for escaping higher order saddle points in
  non-convex optimization.
\newblock \emph{arXiv preprint arXiv:1602.05908}, 2016.

\bibitem[Bandeira et~al.(2016)Bandeira, Boumal, and
  Voroninski]{bandeira2016low}
A.~Bandeira, N.~Boumal, and V.~Voroninski.
\newblock On the low-rank approach for semidefinite programs arising in
  synchronization and community detection.
\newblock \emph{arXiv preprint arXiv:1602.04426}, 2016.

\bibitem[Bhojanapalli et~al.(2015)Bhojanapalli, Kyrillidis, and
  Sanghavi]{bhojanapalli2015dropping}
S.~Bhojanapalli, A.~Kyrillidis, and S.~Sanghavi.
\newblock Dropping convexity for faster semi-definite optimization.
\newblock \emph{arXiv preprint arXiv:1509.03917}, 2015.

\bibitem[Bhojanapalli et~al.(2016)Bhojanapalli, Neyshabur, and
  Srebro]{bhojanapalli2016global}
S.~Bhojanapalli, B.~Neyshabur, and N.~Srebro.
\newblock Global optimality of local search for low rank matrix recovery.
\newblock \emph{arXiv preprint arXiv:1605.07221}, 2016.

\bibitem[Boumal(2016)]{boumal2016nonconvex}
N.~Boumal.
\newblock Nonconvex phase synchronization.
\newblock \emph{arXiv preprint arXiv:1601.06114}, 2016.

\bibitem[Boumal et~al.(2016)Boumal, Voroninski, and Bandeira]{boumal2016non}
N.~Boumal, V.~Voroninski, and A.~Bandeira.
\newblock The non-convex {B}urer-{M}onteiro approach works on smooth
  semidefinite programs.
\newblock \emph{arXiv preprint arXiv:1606.04970}, 2016.

\bibitem[Burer and Monteiro(2003)]{burer2003nonlinear}
S.~Burer and R.~Monteiro.
\newblock A nonlinear programming algorithm for solving semidefinite programs
  via low-rank factorization.
\newblock \emph{Mathematical Programming}, 95\penalty0 (2):\penalty0 329--357,
  2003.

\bibitem[Burer and Monteiro(2005)]{burer2005local}
S.~Burer and R.~Monteiro.
\newblock Local minima and convergence in low-rank semidefinite programming.
\newblock \emph{Mathematical Programming}, 103\penalty0 (3):\penalty0 427--444,
  2005.

\bibitem[Candes(2008)]{candes2008restricted}
E.~Candes.
\newblock The restricted isometry property and its implications for compressed
  sensing.
\newblock \emph{Comptes Rendus Mathematique}, 346\penalty0 (9):\penalty0
  589--592, 2008.

\bibitem[Candes and Plan(2011)]{candes2011tight}
E.~Candes and Y.~Plan.
\newblock Tight oracle inequalities for low-rank matrix recovery from a minimal
  number of noisy random measurements.
\newblock \emph{Information Theory, IEEE Transactions on}, 57\penalty0
  (4):\penalty0 2342--2359, 2011.

\bibitem[Candes et~al.(2011)Candes, Li, Ma, and Wright]{candes2011robust}
E.~Candes, X.~Li, Y.~Ma, and J.~Wright.
\newblock Robust principal component analysis?
\newblock \emph{Journal of the ACM (JACM)}, 58\penalty0 (3):\penalty0 11, 2011.

\bibitem[Chandrasekaran et~al.(2009)Chandrasekaran, Sanghavi, Parrilo, and
  Willsky]{chandrasekaran2009sparse}
V.~Chandrasekaran, S.~Sanghavi, P.~Parrilo, and A.~Willsky.
\newblock Sparse and low-rank matrix decompositions.
\newblock In \emph{Communication, Control, and Computing, 2009. Allerton 2009.
  47th Annual Allerton Conference on}, pages 962--967. IEEE, 2009.

\bibitem[Chen and Wainwright(2015)]{chen2015fast}
Y.~Chen and M.~Wainwright.
\newblock Fast low-rank estimation by projected gradient descent: {G}eneral
  statistical and algorithmic guarantees.
\newblock \emph{arXiv preprint arXiv:1509.03025}, 2015.

\bibitem[Christoffersson(1970)]{christoffersson1970one}
A.~Christoffersson.
\newblock \emph{The one component model with incomplete data}.
\newblock Uppsala., 1970.

\bibitem[Davenport and Romberg(2016)]{davenport2016overview}
M.~Davenport and J.~Romberg.
\newblock An overview of low-rank matrix recovery from incomplete observations.
\newblock \emph{IEEE Journal of Selected Topics in Signal Processing},
  10\penalty0 (4):\penalty0 608--622, 2016.

\bibitem[De~Sa et~al.(2014)De~Sa, Olukotun, and Re]{de2014global}
C.~De~Sa, K.~Olukotun, and C.~Re.
\newblock Global convergence of stochastic gradient descent for some non-convex
  matrix problems.
\newblock \emph{arXiv preprint arXiv:1411.1134}, 2014.

\bibitem[Fazel et~al.(2008)Fazel, Candes, Recht, and
  Parrilo]{fazel2008compressed}
M.~Fazel, E.~Candes, B.~Recht, and P.~Parrilo.
\newblock Compressed sensing and robust recovery of low rank matrices.
\newblock In \emph{Signals, Systems and Computers, 2008 42nd Asilomar
  Conference on}, pages 1043--1047. IEEE, 2008.

\bibitem[Flammia et~al.(2012)Flammia, Gross, Liu, and
  Eisert]{flammia2012quantum}
S.~Flammia, D.~Gross, Y.-K. Liu, and J.~Eisert.
\newblock Quantum tomography via compressed sensing: {E}rror bounds, sample
  complexity and efficient estimators.
\newblock \emph{New Journal of Physics}, 14\penalty0 (9):\penalty0 095022,
  2012.

\bibitem[Ge et~al.(2015)Ge, Huang, Jin, and Yuan]{ge2015escaping}
R.~Ge, F.~Huang, C.~Jin, and Y.~Yuan.
\newblock Escaping from saddle points---online stochastic gradient for tensor
  decomposition.
\newblock In \emph{Proceedings of The 28th Conference on Learning Theory},
  pages 797--842, 2015.

\bibitem[Ge et~al.(2016)Ge, Lee, and Ma]{ge2016matrix}
R.~Ge, J.~Lee, and T.~Ma.
\newblock Matrix completion has no spurious local minimum.
\newblock \emph{arXiv preprint arXiv:1605.07272}, 2016.

\bibitem[Jain et~al.(2015)Jain, Jin, Kakade, and Netrapalli]{jain2015computing}
P.~Jain, C.~Jin, S.~Kakade, and P.~Netrapalli.
\newblock Computing matrix squareroot via non convex local search.
\newblock \emph{arXiv preprint arXiv:1507.05854}, 2015.

\bibitem[Javanmard and Montanari(2013)]{javanmard2013localization}
A.~Javanmard and A.~Montanari.
\newblock Localization from incomplete noisy distance measurements.
\newblock \emph{Foundations of Computational Mathematics}, 13\penalty0
  (3):\penalty0 297--345, 2013.

\bibitem[Jin et~al.(2016)Jin, Kakade, and Netrapalli]{jin2016provable}
C.~Jin, S.~Kakade, and P.~Netrapalli.
\newblock Provable efficient online matrix completion via non-convex stochastic
  gradient descent.
\newblock \emph{arXiv preprint arXiv:1605.08370}, 2016.

\bibitem[Kalev et~al.(2015)Kalev, Kosut, and Deutsch]{kalev2015quantum}
A.~Kalev, R.~Kosut, and I.~Deutsch.
\newblock Quantum tomography protocols with positivity are compressed sensing
  protocols.
\newblock \emph{Nature partner journals (npj) Quantum Information}, 1:\penalty0
  15018, 2015.

\bibitem[Krahmer and Ward(2011)]{krahmer2011new}
F.~Krahmer and R.~Ward.
\newblock New and improved {J}ohnson-{L}indenstrauss embeddings via the
  restricted isometry property.
\newblock \emph{SIAM Journal on Mathematical Analysis}, 43\penalty0
  (3):\penalty0 1269--1281, 2011.

\bibitem[Kyrillidis and Cevher(2014)]{kyrillidis2014matrix}
A.~Kyrillidis and V.~Cevher.
\newblock Matrix recipes for hard thresholding methods.
\newblock \emph{Journal of mathematical imaging and vision}, 48\penalty0
  (2):\penalty0 235--265, 2014.

\bibitem[Lee et~al.(2016)Lee, Simchowitz, Jordan, and Recht]{lee2016gradient}
J.~Lee, M.~Simchowitz, M.~Jordan, and B.~Recht.
\newblock Gradient descent converges to minimizers.
\newblock In \emph{Proceedings of The 29th Conference on Learning Theory},
  2016.

\bibitem[Liu(2011)]{liu2011universal}
Y.-K. Liu.
\newblock Universal low-rank matrix recovery from {P}auli measurements.
\newblock In \emph{Advances in Neural Information Processing Systems}, pages
  1638--1646, 2011.

\bibitem[Liu and Vandenberghe(2009)]{liu2009interior}
Z.~Liu and L.~Vandenberghe.
\newblock Interior-point method for nuclear norm approximation with application
  to system identification.
\newblock \emph{SIAM Journal on Matrix Analysis and Applications}, 31\penalty0
  (3):\penalty0 1235--1256, 2009.

\bibitem[Park et~al.(2016{\natexlab{a}})Park, Kyrillidis, Bhojanapalli,
  Caramanis, and Sanghavi]{park2016provable}
D.~Park, A.~Kyrillidis, S.~Bhojanapalli, C.~Caramanis, and S.~Sanghavi.
\newblock Provable non-convex projected gradient descent for a class of
  constrained matrix optimization problems.
\newblock \emph{arXiv preprint arXiv:1606.01316}, 2016{\natexlab{a}}.

\bibitem[Park et~al.(2016{\natexlab{b}})Park, Kyrillidis, Caramanis, and
  Sanghavi]{park2016finding}
D.~Park, A.~Kyrillidis, C.~Caramanis, and S.~Sanghavi.
\newblock Finding low-rank solutions to matrix problems, efficiently and
  provably.
\newblock \emph{arXiv preprint arXiv:1606.03168}, 2016{\natexlab{b}}.

\bibitem[Park et~al.(2016{\natexlab{c}})Park, Kyrillidis, Caramanis, and
  Sanghavi]{park2016findingb}
D.~Park, A.~Kyrillidis, C.~Caramanis, and S.~Sanghavi.
\newblock Finding low-rank solutions to convex smooth problems via the
  {B}urer-{M}onteiro approach.
\newblock In \emph{54th Annual Allerton Conference on Communication, Control,
  and Computing}, 2016{\natexlab{c}}.

\bibitem[Recht et~al.(2010)Recht, Fazel, and Parrilo]{recht2010guaranteed}
B.~Recht, M.~Fazel, and P.~Parrilo.
\newblock Guaranteed minimum-rank solutions of linear matrix equations via
  nuclear norm minimization.
\newblock \emph{SIAM review}, 52\penalty0 (3):\penalty0 471--501, 2010.

\bibitem[Ruhe(1974)]{ruhe1974numerical}
A.~Ruhe.
\newblock \emph{Numerical computation of principal components when several
  observations are missing}.
\newblock Univ., 1974.

\bibitem[Sun et~al.(2015)Sun, Qu, and Wright]{sun2015complete}
J.~Sun, Q.~Qu, and J.~Wright.
\newblock Complete dictionary recovery over the sphere {I}: Overview and the
  geometric picture.
\newblock \emph{arXiv preprint arXiv:1511.03607}, 2015.

\bibitem[Sun et~al.(2016)Sun, Qu, and Wright]{sun2016geometric}
J.~Sun, Q.~Qu, and J.~Wright.
\newblock A geometric analysis of phase retrieval.
\newblock \emph{arXiv preprint arXiv:1602.06664}, 2016.

\bibitem[Sun and Luo(2015)]{sun2015guaranteed}
R~Sun and Z.-Q. Luo.
\newblock Guaranteed matrix completion via nonconvex factorization.
\newblock In \emph{{IEEE} 56th Annual Symposium on Foundations of Computer
  Science, {FOCS} 2015}, pages 270--289, 2015.

\bibitem[Tu et~al.(2016)Tu, Boczar, Soltanolkotabi, and Recht]{tu2015low}
S.~Tu, R.~Boczar, M.~Soltanolkotabi, and B.~Recht.
\newblock Low-rank solutions of linear matrix equations via {P}rocrustes flow.
\newblock \emph{arXiv preprint arXiv:1507.03566v2}, 2016.

\bibitem[Waters et~al.(2011)Waters, Sankaranarayanan, and
  Baraniuk]{waters2011sparcs}
A.~Waters, A.~Sankaranarayanan, and R.~Baraniuk.
\newblock Spa{RCS}: {R}ecovering low-rank and sparse matrices from compressive
  measurements.
\newblock In \emph{Advances in neural information processing systems}, pages
  1089--1097, 2011.

\bibitem[Wold and Lyttkens(1969)]{wold1969nonlinear}
H.~Wold and E.~Lyttkens.
\newblock Nonlinear iterative partial least squares ({NIPALS}) estimation
  procedures.
\newblock \emph{Bulletin of the International Statistical Institute},
  43\penalty0 (1), 1969.

\bibitem[Yi et~al.(2016)Yi, Park, Chen, and Caramanis]{yi2016rpca}
Xinyang Yi, Dohyung Park, Yudong Chen, and Constantine Caramanis.
\newblock Fast algorithms for robust {PCA} via gradient descent.
\newblock \emph{arXiv preprint arXiv:1605.07784}, 2016.

\bibitem[Zhao et~al.(2015)Zhao, Wang, and Liu]{zhao2015nonconvex}
T.~Zhao, Z.~Wang, and H.~Liu.
\newblock A nonconvex optimization framework for low rank matrix estimation.
\newblock In \emph{Advances in Neural Information Processing Systems 28}, pages
  559--567. 2015.

\bibitem[Zheng and Lafferty(2015)]{zheng2015convergent}
Q.~Zheng and J.~Lafferty.
\newblock A convergent gradient descent algorithm for rank minimization and
  semidefinite programming from random linear measurements.
\newblock In \emph{Advances in Neural Information Processing Systems}, pages
  109--117, 2015.

\bibitem[Zheng and Lafferty(2016)]{zheng2016convergent}
Q.~Zheng and J.~Lafferty.
\newblock Convergence analysis for rectangular matrix completion using
  burer-monteiro factorization and gradient descent.
\newblock \emph{arXiv preprint arXiv:1605.07051}, 2016.

\end{thebibliography}
\bibliographystyle{plainnat}

\end{document}